\documentclass[10pt,twocolumn,letterpaper]{article}

\usepackage[pagenumbers]{cvpr} 

\usepackage{cuted} 

\usepackage{currfile} 

\usepackage{caption} 

\usepackage{amsmath,amsfonts,bm}

\def\eqref#1{equation~\ref{#1}}

\def\1{\bm{1}}

\DeclareMathAlphabet{\mathsfit}{\encodingdefault}{\sfdefault}{m}{sl}
\SetMathAlphabet{\mathsfit}{bold}{\encodingdefault}{\sfdefault}{bx}{n}

\usepackage{amsmath,amssymb,amsfonts}
\usepackage{amsthm}
\usepackage{needspace}
\usepackage{placeins}
\newtheorem{proposition}{Proposition}
\usepackage{mathtools}
\usepackage{bm}
\usepackage{algorithm}
\usepackage{algpseudocode}
\usepackage{booktabs}
\usepackage{multirow}
\usepackage{graphicx}

\definecolor{cvprblue}{rgb}{0.21,0.49,0.74}
\usepackage[pagebackref,breaklinks,colorlinks,allcolors=cvprblue]{hyperref}

\title{Learning the Autoregressive: Direct Score-Difference Estimation with Amortized Future Credit for Video Distillation}

\author{
Zixuan Duan$^{1,2,\dagger}$ \quad
Yabo Chen$^{2,\dagger}$ \quad
Xunzhi Xiang$^{1,\dagger}$ \quad
Xin Zhang$^{2,3}$ \\
Haibin Huang$^{2}$ \quad
Yin Chen$^{2}$ \quad
Chi Zhang$^{2,*}$ \quad
Qi Fan$^{1,*}$ \\
{\small $^{1}$Nanjing University \quad $^{2}$Institute of Artificial Intelligence, China Telecom (TeleAI)}\\
{\small $^{3}$Fudan University}\\
{\small $^{\dagger}$Co-first authors \quad $^{*}$Corresponding authors}
}

\begin{document}
\maketitle
\begin{abstract}
Few-step distillation for autoregressive (AR) video generation commonly first
initializes an AR student and then refines it with distribution matching
distillation (DMD). This refinement
process faces two challenges. First, although initialization already yields a
reasonably capable student, DMD obtains the refinement direction indirectly by
estimating the full student score before forming the teacher--student
correction. Second, detached-history training prevents the generator from
assigning credit to a current chunk for its influence on subsequent chunks.
We introduce Leta, a few-step distillation framework that
addresses these challenges with \mbox{\textsc{ScoreDelta}} and
\mbox{\textsc{FutureCredit}}. \mbox{\textsc{ScoreDelta}} directly estimates
the teacher--student correction from denoising residuals, while
\mbox{\textsc{FutureCredit}} learns an auxiliary cross-chunk update signal
from vector--Jacobian products (VJPs) of next-chunk teacher residuals. Leta
combines these signals into a joint teacher-guided update field in a single
forward pass during generator updates. Teacher evaluations and credit VJPs
are used only for estimator training. For \mbox{\textsc{ScoreDelta}},
we establish population equivalence to score-difference estimation and a
conditional approximation advantage under limited capacity in a fixed-feature
linear model. In our controlled experiments, Leta improves over DMD by
\(0.9\) points on VBench and \(0.66\) on VideoAlign. This advantage persists
at a higher generator update frequency (a generator-to-estimator update ratio
of \(G{:}E=1{:}1\)), where Leta improves
over DMD by \(1.6\) on VBench and \(0.35\) on VideoAlign.
Project page: \url{https://scdzx.github.io/Leta/}.
\end{abstract}

\section{Introduction}
\label{sec:introduction}

Recent advances in video generation have enabled increasingly realistic
and temporally coherent synthesis, yet efficient generation remains
a key challenge~\citep{an2025aiflowperspectivesscenarios,
shao2025aiflownetworkedge,chen2025teleworlddynamicmultimodalsynthesis,
huang2024domainfusion,
wen2025metricsolverslidinganchoredmetric,Zhang_2026_CVPR,
chen2026full4dgeneratingfullscope4d,
zhang2026physomniphysicsgroundedmultiobjectscene,
zhang2026tourphysicsbringingphysicsworld,
xiang2026videoweaveunlockinggeometricconsistency,
huang2026cineweavertrainingfreereferencecontrollablemultishot,
gu2026searchtoworldevaluation3dworld}.
Autoregressive (AR) models support incremental video synthesis by
generating successive chunks conditioned on previously generated
content~\citep{xiang2025macro,xiang2026pathwise,
wang2026directingworldfastautoregressive}.
However, generating each chunk can still require many denoising steps,
motivating the development of few-step AR
generators~\citep{duan2026uncertaintydmdrestoringdiversity,
salimans2022progressive,luo2023latent,wang2023videolcm}.
A common pipeline first initializes such generators and then refines them
through on-policy distribution matching distillation (DMD), using the
student's own generated rollouts so that the conditioning history matches
that encountered at inference~\citep{huang2026self,zhao2026causal}.
DMD guides the student with a teacher--student score difference,
providing a correction signal for distribution matching.
In this setting, two challenges remain: accurately estimating this
correction signal, and accounting for how each generated chunk
influences subsequent generation. We propose Leta to address both challenges
by directly learning the local teacher--student residual and amortizing
teacher-guided credit across chunks.
Figure~\ref{fig:leta-teaser} shows a qualitative example.

\begin{figure}[t]
\centering
\includegraphics[width=\columnwidth]{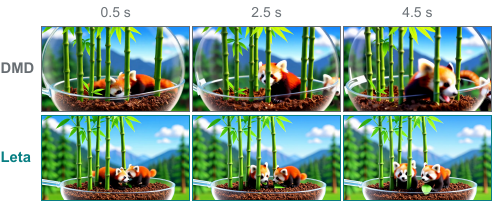}
\caption{DMD and Leta samples at matched timestamps.
Prompt: ``A realistic botanical scene in a Petri dish, where red pandas
explore a miniature bamboo forest.''}
\label{fig:leta-teaser}
\end{figure}

\begin{figure*}[t]
\centering
\includegraphics[width=0.9\textwidth]{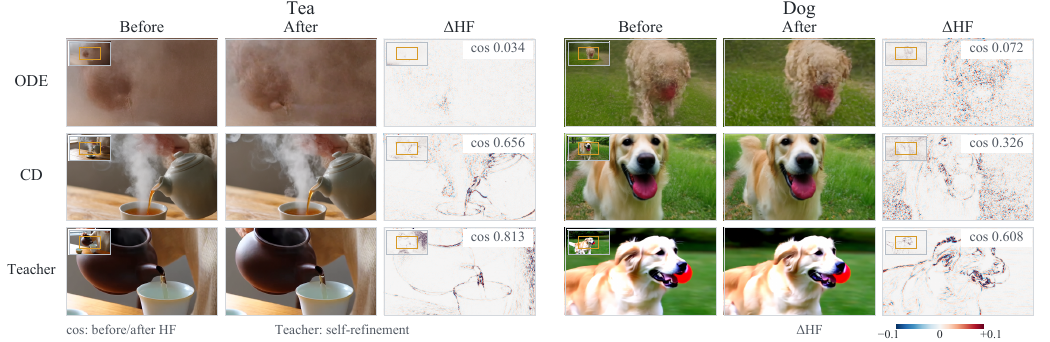}
\caption{Teacher refinement after ODE and CD initialization, with teacher
self-refinement as a reference. In these examples, CD samples show higher
before--after high-frequency (HF) cosine similarity. $\Delta$HF denotes the
HF change; thumbnail boxes mark the enlarged regions.}
\label{fig:initialization-refinement}
\end{figure*}

The first challenge is particularly important in this initialized regime.
Few-step AR students are typically initialized by distilling teacher
denoising trajectories: ODE distillation follows the teacher's probability-flow
trajectory, while Causal Forcing uses an AR teacher and Causal Forcing++ adopts
consistency distillation (CD) for a stronger initialization
\citep{yin2025slow,huang2026self,zhu2026causal,song2023consistency,zhao2026causal}.
To measure how much refinement remains, Figure~\ref{fig:initialization-refinement}
applies the same teacher refinement to ODE- and CD-initialized samples, using
teacher self-refinement as a reference, and compares their before--after
high-frequency (HF) cosine similarity. CD samples change less under refinement,
indicating a stronger starting point and a smaller residual gap. Yet DMD still
estimates the full student score, so estimation errors may mask the subtle
correction needed for further refinement.

To focus learning on the subtle corrections left after initialization, we
propose \textsc{ScoreDelta} to estimate the teacher--student residual directly.
Figure~\ref{fig:leta-concept} (left) contrasts this direct gap estimation
with estimating the full student score before subtraction.
Given a noisy student rollout, the estimator regresses the teacher's denoised
prediction minus the original clean rollout, without separately estimating the full
student score. Proposition~\ref{prop:scoredelta-equivalence} proves that, for
a fixed student distribution and an optimal teacher, the population-optimal
residual estimator recovers the same teacher--student score difference used
by DMD, up to a known timestep scaling. In a fixed-feature linear setting,
Proposition~\ref{prop:finite-capacity} further shows that direct residual
estimation can achieve lower approximation error under the same capacity
constraint when the residual is simpler to represent.

Estimating the score difference is only part of forming an AR generator
update. Each generated chunk also becomes an input to later chunks, so
changing it affects subsequent generation. Backpropagating through the full
rollout incurs a large memory cost, so training typically detaches the
generated history. Although DMD's bidirectional teacher and fake-score model
can use future context, their score predictions do not restore these
cross-chunk gradient paths in the generator. We introduce
\textsc{FutureCredit} to provide teacher-guided auxiliary updates from the next
chunk, augmenting the current chunk's local correction as illustrated
conceptually in Figure~\ref{fig:leta-concept} (right).
A Vector--Jacobian Product (VJP) propagates the next chunk's
teacher residual through a replay of the generator's final denoising step,
providing an auxiliary target. We train an estimator to predict the processed target from the noisy
rollout, making this guidance available through a forward pass during generator
updates.

Together, \textsc{ScoreDelta} and \textsc{FutureCredit} form Leta. On few-step
text-to-video distillation with Wan~\citep{wan2025wan}, Leta improves over DMD
on VBench~\citep{huang2023vbench} Overall and VideoAlign~\citep{liu2026improving}
visual and motion scores. It also achieves higher visual and motion scores than
released
RAVEN~\citep{lu2026ravenrealtimeautoregressivevideo} and
DistillAlign~\citep{li2026distillaligncoordinatingmodecovering} checkpoints,
with competitive VBench performance. The comparisons to the released
checkpoints are reference comparisons, since those checkpoints are not
retrained under our controlled protocol.
Controlled ablations show gains from both components.

Our contributions are threefold:
\begin{itemize}
\item \textbf{Theoretical analysis.}
We provide a theoretical analysis of direct residual estimation,
establishing its connection to DMD's score difference and identifying
conditions under which it offers lower approximation error under
limited estimator capacity.

\item \textbf{\textsc{ScoreDelta}: local correction field.}
We learn the teacher--student residual directly from noisy student rollouts,
focusing the update-field estimator on the remaining discrepancy to support
effective refinement of a well-initialized AR generator.

\item \textbf{\textsc{FutureCredit}: cross-chunk credit field.}
We amortize next-chunk teacher-residual VJPs into a learned auxiliary signal,
allowing generator updates to incorporate future influence under
detached-history training.
\end{itemize}

\begin{figure*}[t]
\centering
\includegraphics[width=\textwidth]{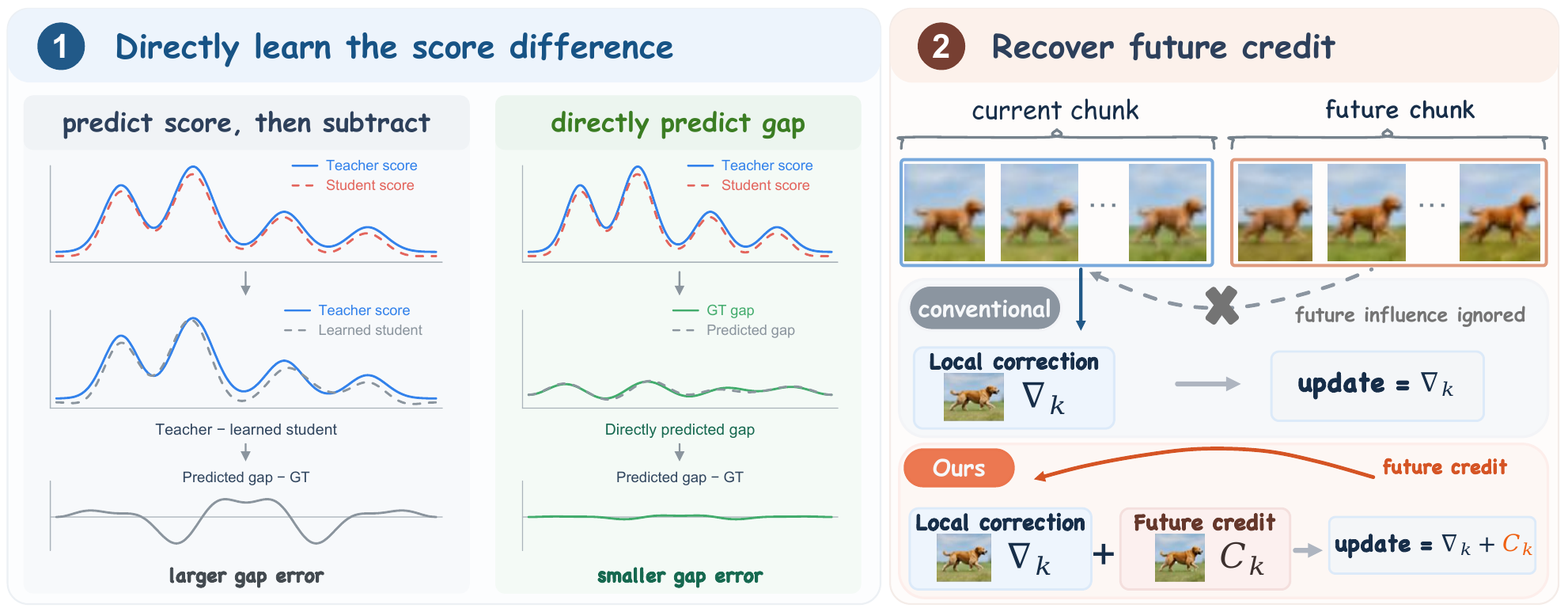}
\caption{Conceptual overview of Leta. \textsc{ScoreDelta} directly estimates
the teacher--student correction, which can be easier to approximate than the
full student score. \textsc{FutureCredit} augments the local correction with
learned guidance from next-chunk teacher-residual VJPs.}
\label{fig:leta-concept}
\end{figure*}

\section{Related Works}

\noindent{\bf Few-Step Distillation.}
Few-step distillation reduces the sampling cost of diffusion models~\citep{lin2026autoregressive,chen2025sana,sauer2024adversarial,geng2026mean,duan2026uncertaintydmdrestoringdiversity}.
Progressive distillation and consistency-based methods further reduce the
number of sampling steps for image and video diffusion models
\citep{salimans2022progressive,luo2023latent,wang2023videolcm}.
Consistency models learn consistent predictions along probability-flow
trajectories~\citep{song2023consistency}, while DMD and DMD2 train
generators through distribution matching using teacher--student score
differences~\citep{yin2024one,yin2024improved}.
rCM combines continuous-time consistency distillation with score
regularization~\citep{zheng2026large},
while AnyFlow uses on-policy flow-map distillation to support
flexible sampling budgets~\citep{gu2026anyflow}.
For AR video generation, Causal Forcing and Causal Forcing++ improve
student initialization through causal ODE and consistency
distillation~\citep{song2023consistency,zhu2026causal,zhao2026causal}.
DistillAlign further combines mode-seeking distribution matching
with a mode-covering consistency
constraint~\citep{li2026distillaligncoordinatingmodecovering}.
Our \textsc{ScoreDelta} focuses on the refinement signal itself,
directly learning the teacher--student correction from noisy student
rollouts.

\noindent{\bf Autoregressive Video Generation.}
Autoregressive video models generate successive chunks conditioned
on previously generated content~\citep{xiang2025macro,xiang2026pathwise,liu2026rolling,yang2026longlive,chen2025teleworlddynamicmultimodalsynthesis,wang2026directingworldfastautoregressive,xiang2026videoweaveunlockinggeometricconsistency,huang2026cineweavertrainingfreereferencecontrollablemultishot,henschel2025streamingt2v,sun2025ar,zhai2025stargen,liu2026streaming}.
CausVid distills a bidirectional diffusion teacher into a few-step
causal generator~\citep{yin2025slow}, while Self Forcing addresses
the training--inference mismatch by training on the student's own
generated histories~\citep{huang2026self}.
RAVEN~\citep{lu2026ravenrealtimeautoregressivevideo} and Self Gradient
Forcing~\citep{zhuang2026self} further enable future losses to supervise historical
context representations.
More broadly, synthetic gradients learn update signals to decouple
neural-network training~\citep{jaderberg2017decoupled}.
Our \textsc{FutureCredit} learns auxiliary guidance from next-chunk
teacher-residual VJPs, making cross-chunk information available
through an estimator forward pass during generator updates with
detached history.

\section{Background}
\label{sec:background}

We first introduce AR generation and self-rollouts, then review on-policy
distribution matching and DMD's alternating updates. 

\paragraph{Autoregressive generation and self-rollouts.}
Let \(P_T(\mathbf{z}\mid c)\) and \(P_G(\mathbf{z}\mid c)\) denote the
distributions represented by the frozen bidirectional teacher and the
autoregressive student, respectively. Conditioned on a prompt \(c\), the AR
student factorizes a video into \(N\) sequential chunks:
\begin{equation}
P_G(\mathbf{z}\mid c)
=\prod_{k=1}^{N}P_G(\mathbf{z}_k\mid\mathbf{z}_{<k},c),
\label{eq:ar-factorization}
\end{equation}
where \(\mathbf{z}_{<k}=(\mathbf{z}_1,\ldots,\mathbf{z}_{k-1})\), with an
empty prefix for \(k=1\). In a self-rollout, each chunk is conditioned on
the student's own generated history:
\begin{equation}
\begin{aligned}
\mathbf{z}_k &\sim P_G(\cdot\mid\mathbf{z}_{<k},c),
\qquad k=1,\ldots,N, \\
\mathbf{z} &=(\mathbf{z}_1,\ldots,\mathbf{z}_N)
\sim P_G(\cdot\mid c).
\end{aligned}
\label{eq:self-rollout}
\end{equation}
The complete rollout is treated as a single video sample for distribution
matching. We use \(\mathbf{z}\) and \(\mathbf{x}_t\) for this on-policy generator
rollout and its post-generation corruption; a separate \(G\) superscript on
sample variables is unnecessary. The \(G\) subscripts in \(P_G\), \(p_{G,t}\),
and \(s_G\) still identify the student distribution.

\paragraph{On-policy distribution matching.}
At noise level \(t\), write the noisy rollout as
\(\mathbf{x}_t=(\mathbf{x}_{1,t},\ldots,\mathbf{x}_{N,t})\), where
\begin{equation}
\begin{aligned}
\mathbf{x}_t
&=
\alpha_t\mathbf{z}
+
\sigma_t\boldsymbol{\epsilon},
\qquad
\boldsymbol{\epsilon}
\sim
\mathcal{N}(\mathbf{0},\mathbf{I}), \\
\mathbf{x}_{k,t}
&=[\mathbf{x}_t]_k
=\alpha_t\mathbf{z}_k+\sigma_t\boldsymbol{\epsilon}_k.
\end{aligned}
\end{equation}
Here \([\cdot]_k\) selects the \(k\)-th chunk, and all chunks share the same
\(t\). This \(t\) indexes post-generation corruption for distribution
matching, distinct from the generator's internal denoising-step index.
The corruption induces the noised student marginal
\(p_{G,t}(\mathbf{x}_t\mid c)\). The frozen diffusion teacher analogously
defines a noised target marginal \(p_{T,t}(\mathbf{x}_t\mid c)\).

At each timestep, distribution matching minimizes the reverse KL divergence
from the student marginal to the teacher marginal:
\begin{equation}
\begin{aligned}
\mathcal{L}_{\mathrm{DM}}^{(t)}
&=\mathrm{KL}\!\left(
p_{G,t}(\cdot\mid c)
\,\middle\|\,
p_{T,t}(\cdot\mid c)
\right) \\
&=\mathbb{E}_{\mathbf{x}_t\sim p_{G,t}}
\Bigl[
\log p_{G,t}(\mathbf{x}_t\mid c)
 -\log p_{T,t}(\mathbf{x}_t\mid c)
\Bigr].
\end{aligned}
\end{equation}
Because the noisy samples are constructed from complete rollouts of the current
student, this is an on-policy objective. Its generator update is governed by
the difference between the teacher and student scores,
where
\begin{equation}
\begin{aligned}
s_T(\mathbf{x}_t,t,c)
&=\nabla_{\mathbf{x}_t}\log p_{T,t}(\mathbf{x}_t\mid c),\\
s_G(\mathbf{x}_t,t,c)
&=\nabla_{\mathbf{x}_t}\log p_{G,t}(\mathbf{x}_t\mid c).
\end{aligned}
\end{equation}
With strong CD initialization, the teacher--student gap can already be small.
Figure~\ref{fig:leta-concept} (left) conceptually illustrates how shared
score structure cancels in the difference, making accurate estimation of the
remaining residual central to further refinement.

\paragraph{DMD training.}
DMD alternates between updating an online fake-score model and updating the
generator. In the fake-score update, the current autoregressive student first
produces complete self-forced rollouts
\(\mathbf{z} \sim P_{G_\theta}(\mathbf{z}\mid c)\).
These on-policy samples are perturbed to obtain
\(\mathbf{x}_t\), and are treated as fixed targets when training the
fake-score model. In clean-prediction notation, this update can be written as
\begin{equation}
\mathcal{L}_{\mathrm{fake}}
=
\mathbb{E}\!\left[
\left\|
D_\psi^G(\mathbf{x}_t,t,c)
-
\operatorname{sg}[\mathbf{z}]
\right\|_2^2
\right],
\end{equation}
where \(D_\psi^G\) denotes the online fake denoiser. At its population optimum,
\(D_\psi^G\) predicts the conditional mean of student samples given their
noisy observations and therefore induces the student score
\(s_G(\mathbf{x}_t,t,c)\). Since the student distribution changes throughout
training, the fake-score model is updated continually on rollouts from the
current student.

In the generator update, both the frozen teacher and the online fake-score
model are held fixed. Fresh on-policy student rollouts are generated, and the
two models are evaluated at the corresponding noisy samples to obtain
\(s_T(\mathbf{x}_t,t,c)\) and \(s_G(\mathbf{x}_t,t,c)\). A
stop-gradient surrogate for the generator update is
\begin{equation}
\mathcal{L}_{G}^{\mathrm{DMD}}
=
\mathbb{E}\!\left[
\left\langle
\operatorname{sg}\!\left[
s_G(\mathbf{x}_t,t,c)
-
s_T(\mathbf{x}_t,t,c)
\right],
\mathbf{x}_t
\right\rangle
\right],
\end{equation}
whose gradient is
\begin{equation}
\nabla_\theta\mathcal{L}_{G}^{\mathrm{DMD}}
=
\mathbb{E}\!\left[
\left(\frac{\partial\mathbf{x}_t}{\partial\theta}\right)^{\!\top}
\left(
s_G(\mathbf{x}_t,t,c)
-
s_T(\mathbf{x}_t,t,c)
\right)
\right].
\end{equation}
In practice, autoregressive training caches the history as fixed key-value pairs,
so gradients propagate only within the current chunk.
This omits the gradient credit for the current chunk's influence on future generation.

\section{Learning the Autoregressive}
\label{sec:ar-update}

Leta learns a joint teacher-guided update field for autoregressive refinement.
A shared bidirectional estimator predicts \textsc{ScoreDelta}, a local
correction, and \textsc{FutureCredit}, a cross-chunk credit signal. Their
complementary targets are optimized together, and their fused direction guides
the generator in one estimator forward pass.
Figure~\ref{fig:method-overview} summarizes this alternating training process.

\begin{figure*}[t]
\centering
\includegraphics[width=\textwidth]{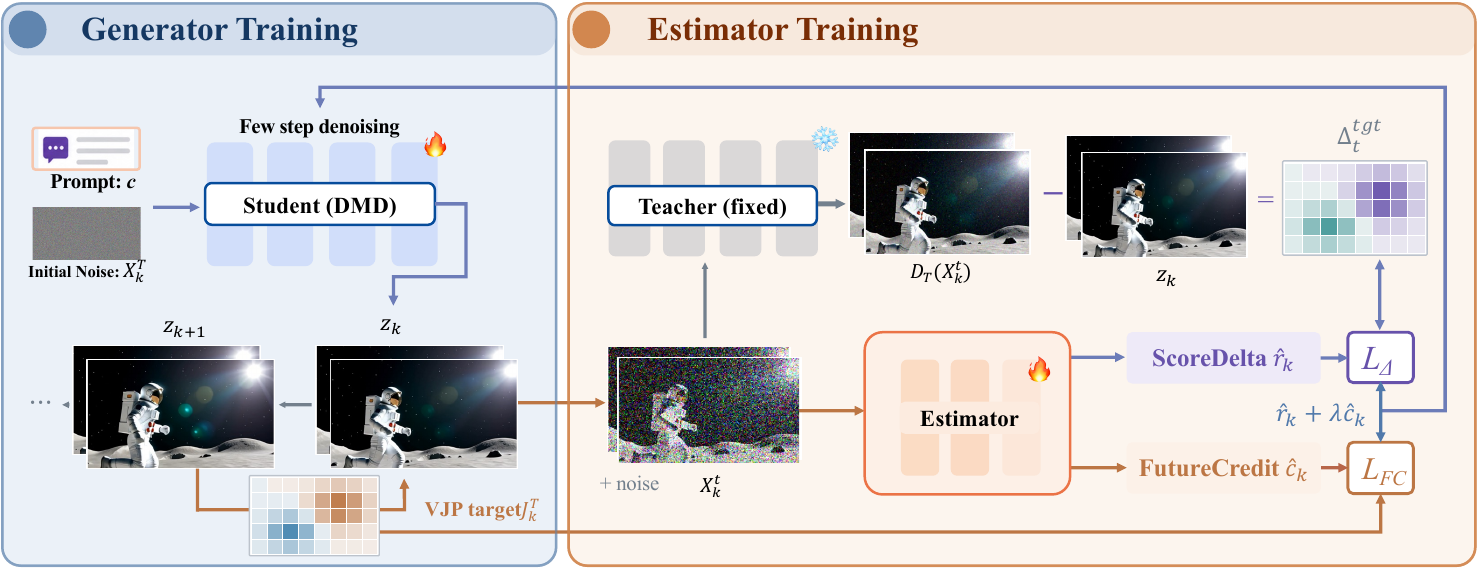}
\caption{Two alternating training phases of Leta. Left: the generator
combines the frozen joint estimator's \textsc{ScoreDelta} correction and
\textsc{FutureCredit} cross-chunk credit. Right: the joint estimator learns a
structured update field under complementary supervision from teacher residuals
and next-chunk VJP targets. Teacher evaluations and VJPs are used only during
estimator training.}
\label{fig:method-overview}
\end{figure*}

\paragraph{\textsc{ScoreDelta}: Local correction field.}
Let \(D_T(\mathbf{x}_t,t,c)\) denote the clean prediction of the frozen
bidirectional teacher. Given a complete on-policy rollout
\(\mathbf{z}\) from the current student and its noisy observation
\(\mathbf{x}_t\), we define the stochastic denoising residual
\begin{equation}
\Delta_t^{\mathrm{tgt}}
=
D_T(\mathbf{x}_t,t,c)-\mathbf{z}.
\end{equation}
The \textsc{ScoreDelta} estimator \(R_\phi\) is learned by regressing this
teacher--student residual:
\begin{equation}
\mathcal{L}_{\Delta}
=
\mathbb{E}
\left[
\left\|
R_\phi(\mathbf{x}_t,t,c)
-
\operatorname{sg}
\left[
\Delta_t^{\mathrm{tgt}}
\right]
\right\|_2^2
\right].
\label{eq:scoredelta-estimator}
\end{equation}
We use the superscript \(\mathrm{tgt}\) for supervision targets and hats for
estimator predictions. The chunk target is
\(\Delta_{k,t}^{\mathrm{tgt}}=[\Delta_t^{\mathrm{tgt}}]_k\).
Unlike a conventional denoiser, \(R_\phi\) is not trained to reconstruct a
complete clean video or to model the student score independently. It directly
predicts the local correction used by the generator update.

\paragraph{Connection to the score difference.}
\begin{proposition}[Population equivalence]
\label{prop:scoredelta-equivalence}
Fix the student distribution and \((t,c)\), with \(\alpha_t,\sigma_t>0\).
Assume \(D_T\) is the optimal teacher denoiser and that \(\mathbf{z}\)
and \(D_T(\mathbf{x}_t,t,c)\) are square-integrable under student sampling.
Let \(D_G^*=\mathbb{E}[\mathbf{z}\mid\mathbf{x}_t,t,c]\) and let
\(R^*\) minimize~\eqref{eq:scoredelta-estimator}, conditioned on \((t,c)\),
over square-integrable functions. Then, almost surely under \(p_{G,t}\),
\begin{equation}
s_T-s_G
=\frac{\alpha_t}{\sigma_t^2}R^*
=\frac{\alpha_t}{\sigma_t^2}\bigl(D_T-D_G^*\bigr),
\label{eq:scoredelta-score-connection}
\end{equation}
where all functions are evaluated at \((\mathbf{x}_t,t,c)\).
\end{proposition}

The proof, the extension to approximate teachers, and the local DMD-gradient
equivalence are given in Appendix~\ref{app:population-equivalence}.

\paragraph{Why direct residual estimation can help at finite capacity.}
Population equivalence does not imply equal approximation error for limited
estimators. If teacher and student share structure, their difference may
require less capacity than the complete student denoiser. We compare the
two denoising losses in clean-prediction units, using the same fixed input
features and linear bottleneck capacity. The indirect method learns a student denoiser \(f\)
and forms \(D_T-f\); \textsc{ScoreDelta} learns the residual directly.

\Needspace{16\baselineskip}
\begin{proposition}[Conditional finite-capacity advantage]
\label{prop:finite-capacity}
Fix \((t,c)\) and a student distribution with finite second moments.
All expectations below use \(\mathbf{x}=\mathbf{x}_t\sim p_{G,t}(\cdot\mid c)\).
Suppose a shared,
fixed feature map \(\mathbf{u}(\mathbf{x})\in\mathbb{R}^m\) satisfies
\(\mathbb{E}[\mathbf{u}\mathbf{u}^{\top}]=I_m\), and
\(D_G^*(\mathbf{x})=A_G\mathbf{u}(\mathbf{x})\),
\(D_T(\mathbf{x})=A_T\mathbf{u}(\mathbf{x})\), with
\(A_G,A_T\in\mathbb{R}^{d\times m}\). Let
\(A_\Delta=A_T-A_G\) and
\(\mathcal{F}_\kappa=\{\mathbf{x}\mapsto M\mathbf{u}(\mathbf{x}):
\operatorname{rank}(M)\leq \kappa\}\), for a rank budget
\(0\leq \kappa\leq\min(d,m)\).
The smallest errors for estimating the same residual \(R^*=D_T-D_G^*\)
are
\begin{equation}
\begin{aligned}
\mathcal{E}_{\mathrm{ind}}^*(\kappa)
&:=\inf_{f\in\mathcal{F}_\kappa}
\mathbb{E}\|D_T-f-R^*\|_2^2
=\sum_{j>\kappa}\sigma_j(A_G)^2, \\
\mathcal{E}_{\Delta}^*(\kappa)
&:=\inf_{R\in\mathcal{F}_\kappa}
\mathbb{E}\|R-R^*\|_2^2
=\sum_{j>\kappa}\sigma_j(A_\Delta)^2,
\end{aligned}
\label{eq:capacity-risks}
\end{equation}
where \(\sigma_j(A)\) are the singular values in decreasing order.
In particular, if
\(\operatorname{rank}(A_\Delta)\leq \kappa<\operatorname{rank}(A_G)\), then
\(\mathcal{E}_{\Delta}^*(\kappa)=0<\mathcal{E}_{\mathrm{ind}}^*(\kappa)\).
\end{proposition}

These errors exclude the common irreducible denoising loss and establish a
conditional approximation advantage in a fixed-feature linear model. This is a
conditional representation result rather than a general guarantee for
arbitrary neural estimators. The spiral toy
example in Figure~\ref{fig:spiral-score-difference} illustrates how cancellation
of shared score structure can benefit direct difference prediction under a
limited feature budget. It is a conditional fixed-density score-regression
illustration, rather than a test of the rank condition or of the full denoising
training procedure. Proofs and experimental details are in
Appendix~\ref{app:finite-capacity} and~\ref{app:spiral-example}.

\begin{figure}[htbp]
\centering
\includegraphics[width=\linewidth]{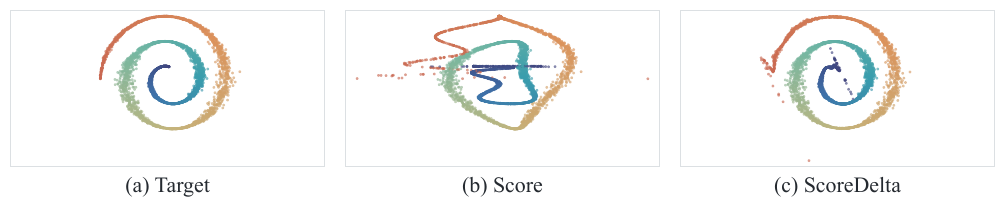}
\caption{Conditional fixed-density score-regression illustration. With matched
trainable parameter budgets,
direct score-difference prediction (c) better preserves the target geometry
(a) than estimating the student score and subtracting it from the teacher
score (b).}
\label{fig:spiral-score-difference}
\end{figure}

\paragraph{Local autoregressive credit.}
Detaching generated history omits how a chunk influences later generation.
To construct an auxiliary signal from this dependence, we replay the next
chunk's final denoising step with its saved noisy input, timestep, and
generator parameters fixed.
Let \(J_k^{\mathrm{rep}}\) be the Jacobian of this clean prediction with
respect to the current history chunk \(\mathbf{z}_k\).

\paragraph{\textsc{FutureCredit}: Cross-chunk credit field.}
For \(k<N\), we propagate the same
stochastic teacher residual \(\Delta_{k+1,t}^{\mathrm{tgt}}
=[D_T(\mathbf{x}_t,t,c)-\mathbf{z}]_{k+1}\) used by \textsc{ScoreDelta}
back to the current chunk:
\begin{equation}
\tilde{\mathbf c}_{k,t}^{\mathrm{VJP}}
=(J_k^{\mathrm{rep}})^{\top}
\operatorname{sg}[\Delta_{k+1,t}^{\mathrm{tgt}}].
\label{eq:future-credit-target}
\end{equation}
This VJP transfers a sample-wise teacher correction to the current chunk. We
use it as a local credit target for a frozen-teacher reconstruction surrogate;
it is a biased approximation to cross-chunk distribution-matching credit, and
we do not claim exact KL-gradient recovery. Thus, \textsc{FutureCredit} is an
auxiliary teacher-guided signal whose effectiveness depends on the local replay
approximation and the mixing weight \(\lambda\).
The \textsc{FutureCredit} estimator \(C_\psi\) predicts a processed version
\(c_{k,t}^{\mathrm{tgt}}\) of this signal to guide generator updates:
\begin{equation}
\mathcal{L}_{\mathrm{FC}}
=\mathbb{E}
\left[\left\|
[C_\psi(\mathbf{x}_t,t,c)]_k
-\operatorname{sg}[c_{k,t}^{\mathrm{tgt}}]\right\|_2^2\right].
\label{eq:future-credit-estimator}
\end{equation}
The two signals are produced by a shared bidirectional estimator and optimized
together under complementary supervision. Target processing and loss sampling are detailed in
Appendix~\ref{app:future-credit-implementation}. Teacher evaluations and VJPs
are used only for estimator training.

\paragraph{Generator update.}
The frozen joint estimator produces both signals in one forward pass.
Write \(\widehat{r}_{k,t}=[R_\phi(\mathbf{x}_t,t,c)]_k\) and
\(\widehat{c}_{k,t}=[C_\psi(\mathbf{x}_t,t,c)]_k\), setting
\(\widehat{c}_{N,t}=0\). We combine them as
\begin{equation}
b_{k,t}=\widehat{r}_{k,t}+\lambda\widehat{c}_{k,t}.
\label{eq:joint-generator-update}
\end{equation}
After scale matching (Appendix~\ref{app:future-credit-implementation}), the
chunk directions form \(\widehat{\mathbf{g}}_t\). The generator retains
gradients only at each chunk's final denoising step, with detached history:
\begin{equation}
\mathcal{L}_{G}
=
-
\mathbb{E}
\left[
w_t
\left\langle
\operatorname{sg}
\left[
\widehat{\mathbf{g}}_t
\right],
\mathbf{z}
\right\rangle
\right].
\label{eq:ar-update-generator}
\end{equation}
The \textsc{ScoreDelta}-only update's connection to DMD is derived in
Appendix~\ref{app:local-gradient-equivalence}.

\begin{algorithm}[H]
\caption{Leta training overview}
\label{alg:leta-training}
\small
\begin{algorithmic}[1]
\Require Generator \(G_\theta\), frozen teacher \(D_T\), and
joint estimator \((R_\phi,C_\psi)\).
\While{training}
    \State Sample \(c,t,\boldsymbol{\epsilon}\) and generate a fresh rollout \(\mathbf{z}\).
    \State \(\mathbf{x}_t \gets \alpha_t\mathbf{z}+\sigma_t\boldsymbol{\epsilon}\).
    \If{an estimator update is scheduled}
        \State \(\Delta_t^{\mathrm{tgt}} \gets D_T(\mathbf{x}_t,t,c)-\mathbf{z}\).
        \State Sample \(k<N\); compute \(\tilde{\mathbf c}_{k,t}^{\mathrm{VJP}}\)
        by next-chunk replay using~\eqref{eq:future-credit-target}.
        \State Project and scale-match to obtain \(c_{k,t}^{\mathrm{tgt}}\).
        \State Update the joint estimator with
        \(\mathcal{L}_{\Delta}+\mathcal{L}_{\mathrm{FC}}\), keeping \(\theta\) fixed.
    \Else \Comment{Generator update}
        \State Obtain \(\widehat r_{k,t},\widehat c_{k,t}\) for all chunks
        in one frozen-estimator forward pass.
        \State Set \(\widehat c_{N,t}=0\);
        \(b_{k,t}\gets\widehat r_{k,t}+\lambda\widehat c_{k,t}\).
        \State Scale-match to obtain \(\widehat{\mathbf g}_t\);
        update \(\theta\) with \(\mathcal{L}_G\) in~\eqref{eq:ar-update-generator}.
    \EndIf
\EndWhile
\end{algorithmic}
\end{algorithm}

\paragraph{Training procedure.}
Algorithm~\ref{alg:leta-training} summarizes the alternating updates: the
joint estimator learns both fields from teacher-derived targets and then
supplies their combined direction for generator updates. Teacher evaluations
and credit VJPs are used only during estimator training.




\Needspace{8\baselineskip}

\section{Experiment}
We perform final-stage on-policy few-step distillation using the
CD-initialized causal Wan2.1-T2V-1.3B model provided by Causal
Forcing++~\citep{zhao2026causal}. Our controlled training comparisons include
DMD, \textsc{ScoreDelta}, and Leta. Unless otherwise specified, these runs
alternate one generator update with four estimator updates (\(G:E=1:4\)),
and use a learning rate of
\(1\times10^{-6}\) for the generator and \(2\times10^{-6}\) for the estimator,
with a batch size of 16 on 16 NVIDIA H100 GPUs.
Following the Self-Forcing setting~\citep{huang2026self}, each rollout
consists of \(N{=}7\) chunks of 3 latent frames each (21 latent frames in
total), with four denoising steps for each generated chunk; all other
generation, rollout, and training-data details follow the public
Self-Forcing codebase.
Unless otherwise specified, the auxiliary estimator is initialized from the
bidirectional multi-step Wan2.1-T2V-1.3B model~\citep{wan2025wan}; when a
teacher score or clean prediction is required, the same pretrained model is
kept frozen and used as the teacher. Our standard 1:4 runs use 1,200 total
optimizer updates; the 1:1 ablation uses 600 total updates.
Our trained variants share the same causal
initialization, training data, self-forcing rollout procedure, and
sampling-step configuration. RAVEN~\citep{lu2026ravenrealtimeautoregressivevideo}
and DistillAlign~\citep{li2026distillaligncoordinatingmodecovering} are evaluated
using their publicly released checkpoints.

\paragraph{Evaluation protocol.}
We evaluate all methods on VBench~\citep{huang2023vbench}
using five shared, fixed generation seeds and report the averaged
results. We additionally evaluate the
generated videos with VideoAlign~\citep{liu2026improving}, a VLM-based metric
for visual quality, motion quality, and text alignment.
Using the same seeds for all controlled variants yields paired generations,
so the reported values are paired averages and are less affected by variation
in the sampled noise.

\begin{figure*}[t]
\centering
\includegraphics[width=0.9\textwidth]{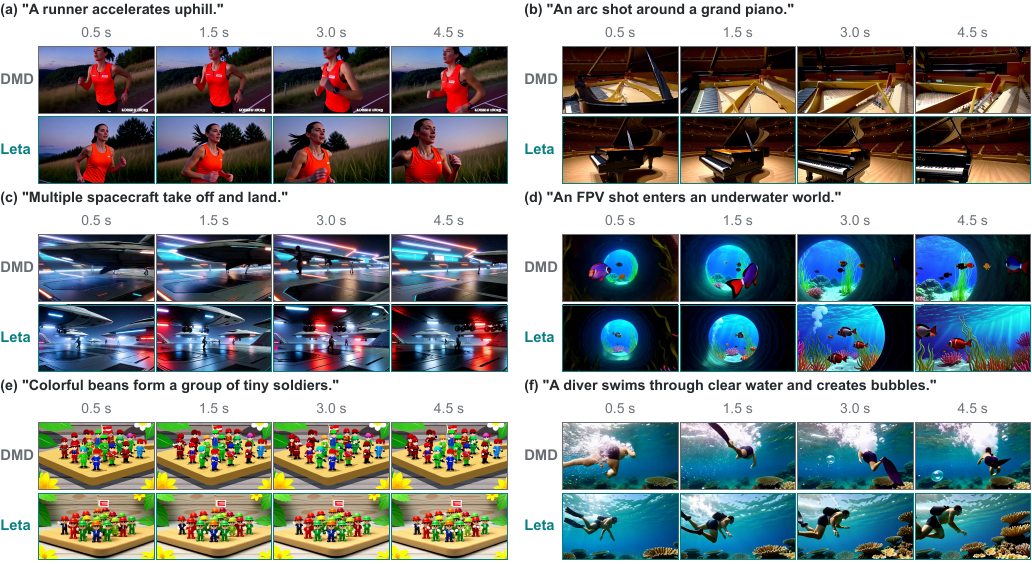}
\caption{DMD and Leta samples at matched timestamps. Each prompt is shown
above its corresponding pair.}
\label{fig:leta-qualitative-comparison}
\end{figure*}

\begin{table}[tb]
\centering
\caption{Comparison on VBench and VideoAlign. VQ, MQ, and TA denote visual
quality, motion quality, and text alignment.}
\label{tab:overall_eval}
\parbox{\columnwidth}{\footnotesize
\(\dagger\) Two-stage training: 300 steps of DMD with a Wan-14B teacher,
followed by 900 steps of Leta with a Wan-1.3B teacher.}
\par\vspace{4pt}
\small
\setlength{\tabcolsep}{3.5pt}
\resizebox{\columnwidth}{!}{%
\begin{tabular}{llccccccc}
\toprule
& & \multicolumn{3}{c}{VBench} & \multicolumn{4}{c}{VideoAlign} \\
\cmidrule(lr){3-5}\cmidrule(lr){6-9}
Model & \shortstack{Teacher\\model}
& Quality & Semantic & Overall
& VQ & MQ & TA & Overall \\
\midrule
DMD & Wan-14B
& 85.1 & 80.2 & 84.1
& 0.141 & 0.174 & 1.234 & 1.548 \\
DMD & Wan-1.3B
& 84.5 & 80.0 & 83.6
& 0.261 & 0.197 & 1.234 & 1.692 \\
RAVEN & Wan-14B
& 85.6 & 80.9 & 84.6
& 0.317 & 0.415 & 1.318 & 2.049 \\
DistillAlign & Wan-1.3B
& 85.5 & 80.7 & 84.6
& 0.329 & 0.216 & 1.231 & 1.776 \\
Leta & Wan-1.3B
& 85.6 & 80.2 & 84.5
& 0.604 & 0.579 & 1.165 & 2.348 \\
Leta & Wan-14B\(^{\dagger}\)
& 85.9 & 80.3 & 84.8
& 0.426 & 0.586 & 1.166 & 2.174 \\
\bottomrule
\end{tabular}%
}
\end{table}

\FloatBarrier

\paragraph{Main results.}
Table~\ref{tab:overall_eval} compares Leta with DMD and the released RAVEN
and DistillAlign checkpoints. Under the controlled setting, Leta improves
VBench Overall and VideoAlign visual and motion quality over DMD. Compared
with the released checkpoints, Leta achieves higher VideoAlign visual and
motion scores and competitive VBench performance, with gains concentrated in
appearance and motion rather than text alignment. The released-checkpoint
comparisons should be interpreted as reference comparisons because those
models are not retrained under our protocol.
The paired samples in Figure~\ref{fig:leta-qualitative-comparison} also show
better preservation of object structure, scene composition, and subject
consistency.

\begin{table}
\centering
\vspace{-2pt}
\caption{Ablations of Leta components and generator-to-estimator update ratio.
The upper block uses the standard \(G{:}E=1{:}4\); the lower block
uses \(G{:}E=1{:}1\) with 600 total optimizer updates.}
\label{tab:ablations}
\small
\setlength{\tabcolsep}{5pt}
\makebox[\columnwidth][l]{%
\resizebox{\columnwidth}{!}{%
\begin{tabular}{lccccccc}
\toprule
& \multicolumn{3}{c}{VBench} & \multicolumn{4}{c}{VideoAlign} \\
\cmidrule(lr){2-4}\cmidrule(lr){5-8}
Model
& Quality & Semantic & Overall
& VQ & MQ & TA & Overall \\
\midrule
\multicolumn{8}{l}{\textit{Component ablation (\(G{:}E=1{:}4\))}} \\
\midrule
DMD
& 84.5 & 80.0 & 83.6
& 0.261 & 0.197 & 1.234 & 1.692 \\
\textsc{ScoreDelta}
& 85.2 & 79.7 & 84.1
& 0.511 & 0.445 & 1.168 & 2.124 \\
Leta
& 85.6 & 80.2 & 84.5
& 0.604 & 0.579 & 1.165 & 2.348 \\
\midrule
\multicolumn{8}{l}{\textit{Higher generator update frequency (\(G{:}E=1{:}1\))}} \\
\midrule
DMD (1:1)
& 83.3 & 79.7 & 82.6
& 0.302 & 0.125 & 1.263 & 1.689 \\
\textsc{ScoreDelta} (1:1)
& 85.1 & 79.9 & 84.1
& 0.352 & 0.408 & 1.191 & 1.951 \\
Leta (1:1)
& 85.3 & 79.9 & 84.2
& 0.434 & 0.437 & 1.173 & 2.043 \\
\bottomrule
\end{tabular}%
}%
\vspace{-2pt}
}
\end{table}

\begin{table}
\centering
\caption{Effect of the future-credit weight $\lambda$.}
\label{tab:ablation_future_credit_weight}
\small
\setlength{\tabcolsep}{5pt}
\resizebox{\columnwidth}{!}{%
\begin{tabular}{lccccccc}
\toprule
& \multicolumn{3}{c}{VBench} & \multicolumn{4}{c}{VideoAlign} \\
\cmidrule(lr){2-4}\cmidrule(lr){5-8}
Model
& Quality & Semantic & Overall
& VQ & MQ & TA & Overall \\
\midrule
ScoreDelta
& 85.2 & 79.7 & 84.1
& 0.511 & 0.445 & 1.168 & 2.124 \\
\midrule
Leta ($\lambda=0$)
& 85.2 & 79.3 & 84.0 & 0.482 & 0.546 & 1.160 & 2.188 \\
Leta ($\lambda=0.25$)
& 85.6 & 80.2 & 84.5 & 0.604 & 0.579 & 1.165 & 2.348 \\
Leta ($\lambda=0.5$)
& 85.2 & 79.9 & 84.1 & 0.498 & 0.538 & 1.191 & 2.228 \\
Leta ($\lambda=1$)
& 85.7 & 79.1 & 84.4 & 0.488 & 0.476 & 1.176 & 2.139 \\
\bottomrule
\end{tabular}%
}
\vspace{-2pt}
\end{table}

\paragraph{Ablations.}
Table~\ref{tab:ablations} evaluates the contributions of direct residual
estimation and future credit. \textsc{ScoreDelta} already improves over DMD,
and adding \textsc{FutureCredit} yields further gains in VBench and VideoAlign.
The same ordering holds under the balanced \(G{:}E=1{:}1\) schedule.

\paragraph{Effect of future-credit weight.}
Table~\ref{tab:ablation_future_credit_weight} studies the future-credit weight
\(\lambda\). A moderate value (\(\lambda=0.25\)) performs best, while removing
future credit yields performance close to \textsc{ScoreDelta} and larger weights
can degrade the score.

\section{Conclusion}

We presented Leta, a distillation framework for few-step AR video
generation that combines \textsc{ScoreDelta}'s direct local correction with
\textsc{FutureCredit}'s cross-chunk credit in a shared estimator.
Experiments show improved visual and motion quality over DMD across
evaluated settings. Nevertheless, current AR video generation lacks
explicit planning for future motion and scene evolution, which may
lead to unnatural dynamics in complex scenes and error accumulation
over long horizons.

{\small
\bibliographystyle{ieeenat_fullname}
\bibliography{main}
}

\clearpage
\appendix

\section{Population Equivalence of Residual and Score Estimation}
\label{app:population-equivalence}

Throughout this appendix, \(t\) denotes the Gaussian forward-corruption
noise level of the complete rollout. It is distinct from the autoregressive
chunk index and the generator's internal solver-step index.
Fix the student distribution and condition on \((t,c)\). We suppress these arguments.
Write \(a=\alpha_t>0\), \(s=\sigma_t>0\),
\(\mathbf{z}\sim P_G(\cdot\mid c)\), and
\(\mathbf{x}=a\mathbf{z}+s\boldsymbol{\epsilon}\), where
\(\boldsymbol{\epsilon}\sim\mathcal{N}(0,I)\) is independent of
\(\mathbf{z}\). Expectations are under student sampling unless stated
otherwise, and all regression targets are square-integrable.

\subsection{Proof of Proposition~\ref{prop:scoredelta-equivalence}}

\begin{proof}
Let \(Y=\Delta_t^{\mathrm{tgt}}=D_T(\mathbf{x})-\mathbf{z}\). Since the frozen teacher is a
deterministic function of the regression input,
\begin{equation}
\mathbb{E}[Y\mid\mathbf{x}]
=D_T(\mathbf{x})-\mathbb{E}[\mathbf{z}\mid\mathbf{x}]
=D_T(\mathbf{x})-D_G^*(\mathbf{x})=:R^*(\mathbf{x}).
\label{eq:optimal-student-denoiser}
\end{equation}
For any square-integrable predictor \(R\), conditioning on \(\mathbf{x}\)
eliminates the cross term and gives
\begin{equation}
\begin{aligned}
\mathbb{E}\|R(\mathbf{x})-Y\|_2^2
&=\mathbb{E}\|R(\mathbf{x})-R^*(\mathbf{x})\|_2^2+\mathcal{N}_t, \\
\mathcal{N}_t
&:=\mathbb{E}\|\mathbf{z}-D_G^*(\mathbf{x})\|_2^2.
\end{aligned}
\label{eq:residual-risk-decomposition}
\end{equation}
The second term is independent of \(R\). Therefore \(R^*\) is the unique
population minimizer up to sets of zero student probability. Individual
training targets are noisy; their conditional mean is the desired residual.

For either clean distribution \(P\), write its noised density as
\begin{equation}
p_t(\mathbf{x})
=\int \varphi_s(\mathbf{x}-a\mathbf{z})\,P(d\mathbf{z}),
\end{equation}
where \(\varphi_s\) is the density of \(\mathcal{N}(0,s^2I)\).
The Gaussian kernel permits differentiation under the integral, and
\(p_t(\mathbf{x})>0\). Thus
\begin{equation}
\begin{aligned}
\nabla_{\mathbf{x}}\log p_t(\mathbf{x})
&=\frac{\int -(\mathbf{x}-a\mathbf{z})s^{-2}
\varphi_s(\mathbf{x}-a\mathbf{z})\,P(d\mathbf{z})}
{p_t(\mathbf{x})} \\
&=\frac{a\,\mathbb{E}_{P}[\mathbf{z}\mid\mathbf{x}]-\mathbf{x}}{s^2}.
\end{aligned}
\label{eq:gaussian-denoiser-score}
\end{equation}
This is the Gaussian denoiser--score identity (Tweedie's formula;
\citealp{efron2011tweedie}). If \(D_T\) is the optimal teacher denoiser,
applying it to teacher and student cancels the common \(-\mathbf{x}/s^2\)
term. Hence
\begin{equation}
s_T(\mathbf{x})-s_G(\mathbf{x})
=\frac{a}{s^2}\bigl(D_T(\mathbf{x})-D_G^*(\mathbf{x})\bigr)
=\frac{a}{s^2}R^*(\mathbf{x}),
\end{equation}
which proves the proposition.
\end{proof}

\paragraph{Approximate frozen teacher.}
For any fixed teacher predictor, let
\(\widetilde{s}_T=(aD_T-\mathbf{x})/s^2\) denote the score field it
provides. The same conditional-mean argument gives
\begin{equation}
\frac{a}{s^2}R^*
=\widetilde{s}_T-s_G.
\label{eq:approximate-teacher-equivalence}
\end{equation}
Thus, at the population optimum, \textsc{ScoreDelta} and DMD using the
same teacher recover the same correction. Both inherit the teacher's error
relative to the true score \(s_T\).

\subsection{Equivalence of the local ScoreDelta gradient}
\label{app:local-gradient-equivalence}

Let \(J_\theta\) denote the Jacobian of \(\mathbf{z}\) with respect to
\(\theta\) retained by the generator computation graph, which may use
gradient truncation. Both methods must use the same graph. Holding the
sampled noise fixed gives
\(\partial\mathbf{x}_t/\partial\theta=\alpha_tJ_\theta\).
The DMD stop-gradient surrogate in the main text therefore has gradient
\begin{equation}
\begin{aligned}
\nabla_\theta\mathcal{L}_{G}^{\mathrm{DMD}}
&=\mathbb{E}\!\left[
\alpha_tJ_\theta^{\top}(s_G-s_T)\right] \\
&=-\mathbb{E}\!\left[
\frac{\alpha_t^2}{\sigma_t^2}J_\theta^{\top}R^*\right].
\end{aligned}
\label{eq:dmd-gradient-residual-form}
\end{equation}
With \(\lambda=0\), the \textsc{ScoreDelta} surrogate is
\(-\mathbb{E}[w_t\langle\operatorname{sg}[R^*],\mathbf{z}\rangle]\),
whose gradient is exactly the last expression when
\(w_t=\alpha_t^2/\sigma_t^2\). For an approximate teacher, replace
\(s_T\) by \(\widetilde{s}_T\) in both methods. Averaging over \((t,c)\)
preserves the equality when sampling and weights agree and the gradients
are integrable. This equivalence concerns the local update at the population
optimum before RMS matching; it does not include the added
\textsc{FutureCredit} term. Our experiments use
\(w_t=\alpha_t^2/\sigma_t^2\), together with the processed direction specified
in Appendix~\ref{app:future-credit-implementation}.

\section{A Conditional Advantage under Finite Estimator Capacity}
\label{app:finite-capacity}

\subsection{Proof of Proposition~\ref{prop:finite-capacity}}

\begin{proof}
We again fix \((t,c)\), suppress these arguments, and write
\(\mathbf{z}\). Dependence of the feature map, target matrices,
and risks on \((t,c)\) is implicit, while \(\kappa\) denotes the rank budget.
The proof has three steps.

\emph{Step 1: separate approximation error from target noise.}
The conditional-mean decomposition gives
\begin{equation}
\begin{aligned}
\mathcal{L}_{\mathrm{fake}}(f)
&=\mathbb{E}\|f-D_G^*\|_2^2+\mathcal{N}_t, \\
\mathcal{L}_{\Delta}(R)
&=\mathbb{E}\|R-R^*\|_2^2+\mathcal{N}_t,
\end{aligned}
\label{eq:common-noise-floor}
\end{equation}
with the same \(\mathcal{N}_t=\mathbb{E}\|\mathbf{z}-D_G^*(\mathbf{x})\|_2^2\)
in both losses. For the indirect residual, the pointwise identity
\begin{equation}
(D_T-f)-R^*=D_G^*-f
\label{eq:no-absolute-error-amplification}
\end{equation}
shows that its squared error is exactly the denoiser's approximation error.
Subtraction does not amplify absolute error. Hence both methods can be
compared on the same residual target after removing the common noise term.

\emph{Step 2: express finite capacity as matrix rank.}
The class \(\mathcal{F}_\kappa\) is implemented by \(M=VU^{\top}\), with
\(V\in\mathbb{R}^{d\times \kappa}\) and \(U\in\mathbb{R}^{m\times \kappa}\).
Both estimators therefore have a linear bottleneck of width \(\kappa\) and at most
\(\kappa(d+m)\) trainable parameters, using the same fixed features. Let
\(B=M-A\). Since \(\mathbb{E}[\mathbf{u}\mathbf{u}^{\top}]=I_m\),
\begin{equation}
\mathbb{E}\|B\mathbf{u}\|_2^2
=\operatorname{tr}(BB^{\top})
=\|B\|_F^2.
\label{eq:whitened-matrix-risk}
\end{equation}
Thus the two approximation problems have target matrices \(A_G\) and
\(A_\Delta=A_T-A_G\), respectively.

\emph{Step 3: compute the best rank-\(\kappa\) error.}
Let \(\Pi\) be the orthogonal projection onto the column space of \(M\).
Then \((I-\Pi)M=0\), and orthogonality gives
\begin{equation}
\begin{aligned}
\|M-A\|_F^2
&=\|M-\Pi A\|_F^2+\|(I-\Pi)A\|_F^2 \\
&\geq\operatorname{tr}(AA^{\top})-
\operatorname{tr}(\Pi AA^{\top})
\geq\sum_{j>\kappa}\sigma_j(A)^2.
\end{aligned}
\label{eq:rank-approximation-bound}
\end{equation}
For the last inequality, express \(\Pi\) in an eigenbasis of
\(AA^{\top}\). Its diagonal entries lie in \([0,1]\) and sum to at most
\(\kappa\); their weighted sum cannot exceed the sum of the largest \(\kappa\)
eigenvalues. Taking \(\Pi\) to project onto those eigenvectors and
\(M=\Pi A\) attains equality. Therefore
\begin{equation}
\inf_{\operatorname{rank}(M)\leq \kappa}\|M-A\|_F^2
=\sum_{j>\kappa}\sigma_j(A)^2.
\label{eq:optimal-rank-risk}
\end{equation}
Applying this formula to \(A_G\) and \(A_\Delta\) proves
\eqref{eq:capacity-risks}. If
\(\operatorname{rank}(A_\Delta)\leq \kappa<\operatorname{rank}(A_G)\), the
residual fits exactly, while at least one nonzero singular value of the
complete denoiser must be discarded. This proves the strict gap.
\end{proof}

The rank condition is sufficient, not necessary: the same advantage holds
whenever the residual has a smaller singular-value tail. Under an exact
teacher, converting both residual estimates to score differences multiplies
both squared errors by \(\alpha_t^2/\sigma_t^4\), preserving the comparison
at the fixed timestep.

\subsection{A Gaussian example}
\label{app:gaussian-capacity}

The following construction shows that the rank condition can hold for valid
teacher and student distributions under the paper's Gaussian corruption.
Let \(d\geq2\), \(0<\delta<v\), and let \(e_d\) be the last coordinate
vector. Choose clean distributions
\begin{equation}
\begin{aligned}
P_G&=\mathcal{N}(0,C_G), & C_G&=vI,\\
P_T&=\mathcal{N}(0,C_T), & C_T&=vI-\delta e_de_d^{\top}.
\end{aligned}
\label{eq:gaussian-clean-covariances}
\end{equation}
Both covariances are positive definite. The distributions share their first
\(d-1\) coordinates and differ only in the variance of the last coordinate.

Fix \(a=\alpha_t>0\), \(s=\sigma_t>0\), and write
\(v_G=a^2v+s^2\), \(v_T=a^2(v-\delta)+s^2\).
For student samples \(\mathbf{x}=a\mathbf{z}+s\boldsymbol{\epsilon}\),
the common input \(\mathbf{u}=\mathbf{x}/\sqrt{v_G}\) satisfies
\(\mathbf{u}\sim\mathcal{N}(0,I)\). Gaussian conditioning gives
\begin{equation}
D_G^*(\mathbf{x})=\frac{av}{v_G}\mathbf{x},
\qquad
D_T(\mathbf{x})=aC_T(a^2C_T+s^2I)^{-1}\mathbf{x}.
\end{equation}
In the shared coordinates \(\mathbf{u}\), the complete denoiser and residual
therefore have matrices
\begin{equation}
A_G=\frac{av}{\sqrt{v_G}}I,
\qquad
A_\Delta=-\frac{as^2\delta}{v_T\sqrt{v_G}}e_de_d^{\top}.
\label{eq:gaussian-residual-rank}
\end{equation}
Indeed, the first \(d-1\) denoiser coefficients cancel, while the last
coefficient of \(A_T-A_G\) is
\begin{equation}
\begin{aligned}
&\frac{a(v-\delta)\sqrt{v_G}}{v_T}
-\frac{av}{\sqrt{v_G}} \\
&\quad =
\frac{a\{(v-\delta)v_G-vv_T\}}{v_T\sqrt{v_G}} \\
&\quad =
-\frac{as^2\delta}{v_T\sqrt{v_G}} .
\end{aligned}
\end{equation}
Thus \(A_G\) has rank \(d\), whereas \(A_\Delta\) has rank \(1\).
All singular values of \(A_G\) equal \(av/\sqrt{v_G}\), so for every
integer \(1\leq \kappa<d\), Proposition~\ref{prop:finite-capacity} yields
\begin{equation}
\mathcal{E}_{\mathrm{ind}}^*(\kappa)
=(d-\kappa)\frac{a^2v^2}{v_G}>0,
\qquad
\mathcal{E}_{\Delta}^*(\kappa)=0.
\label{eq:gaussian-capacity-gap}
\end{equation}
The zero is an approximation error: both optimal training losses still
include the same positive denoising noise
\(\mathcal{N}_t=dvs^2/v_G\). When \(\kappa=d\), both approximation errors
vanish. The example isolates a finite-capacity benefit from cancelling
shared structure, without requiring a claim about optimization speed.

\subsection{What the capacity comparison assumes}
\label{app:capacity-scope}

Although the trainable estimators have the same capacity, their implied
student denoisers belong to different function classes:
\begin{equation}
\begin{aligned}
\widehat{D}_G&\in\mathcal{F}_\kappa
&&\text{(complete denoising)},\\
\widehat{D}_G&=D_T-R\in D_T-\mathcal{F}_\kappa
&&\text{(residual estimation)}.
\end{aligned}
\end{equation}
Here the teacher enters \textsc{ScoreDelta} through the regression targets;
\(D_T-R\) is the implied denoiser, not an extra teacher evaluation during
the generator update. If the complete denoiser instead uses the same teacher
skip parameterization \(D_T-R_\phi\), then
\begin{equation}
\|(D_T-R_\phi)-\mathbf{z}\|_2^2
=\|R_\phi-(D_T-\mathbf{z})\|_2^2.
\label{eq:teacher-skip-equivalence}
\end{equation}
These losses and their parameter gradients are identical for the same
samples and weights. The capacity advantage is therefore conditional on
the chosen function classes and the complexity of their targets; changing
the target's name alone does not create an advantage.

\Needspace{32\baselineskip}
\subsection{Nonlinear spiral illustration}
\label{app:spiral-example}

Figure~\ref{fig:spiral-score-difference} uses a fixed-density score
regression problem. For \((x,y)\in\mathbb{R}^2\), define the student
as a narrow Gaussian strip with a bounded density modulation:
\begin{equation}
\begin{aligned}
q_\rho(x,y)&\propto
\exp\!\left[-\frac{x^2}{2}-\frac{y^2}{2\eta^2}
            +\frac{\rho}{\omega}\cos(\omega x)\right],\\
&\qquad \eta=0.015,\quad \omega=12.
\end{aligned}
\label{eq:spiral-student-density}
\end{equation}
Let \(u=\Phi(x)\), where \(\Phi\) is the standard normal CDF, and set
\(r(u)=0.35+1.65u\), \(\theta(u)=\pi/2+4.5\pi u\),
\(A(x)=r(u)\cos\theta(u)\), and \(B(x)=r(u)\sin\theta(u)\).
The teacher is defined through a log-density ratio:
\begin{equation}
\begin{aligned}
p_\rho(x,y)&\propto q_\rho(x,y)\exp[F(x)+yB(x)],\\
F(x)&=\int_0^x A(v)\,dv,\\
\Delta(x,y)&=\nabla\log p_\rho-\nabla\log q_\rho
\\
&\qquad=\bigl(A(x)+yB'(x),\,B(x)\bigr).
\end{aligned}
\label{eq:spiral-score-difference}
\end{equation}
Bounded \(A,B\) make \(F(x)+yB(x)\) grow at most linearly, so Gaussian tails
ensure normalizability. Differentiating the log-density ratio gives the
stated difference. At \(y=0\), it traces a spiral; the small transverse
variation gives the output distribution thickness.

Both estimators use 51 identical fixed features: \(1,x,y/\eta\),
12 centers at each of two Gaussian-RBF bandwidths (24 RBFs total), and their
24 interactions with \(y/\eta\).
The 12 centers are equally spaced between the training inputs' 0.2\%
and 99.8\% quantiles in \(x\); bandwidths are 0.65 and 1.3 times their
spacing. Each two-dimensional readout thus has 102 trainable parameters.
The displayed paper setting fixes these 12 centers and \(\rho=0.6\); the
reported numbers below correspond to this figure configuration.
We regress exact \(s_G=\nabla\log q_\rho\) or \(\Delta\) on the same
8,192 training inputs. Affine terms are unpenalized; the remaining
features are projected off the affine subspace and standardized using
training data. Each route independently selects ridge strength from
\(\{0,10^{-10},10^{-8},10^{-6},10^{-4},10^{-2},1\}\) on 4,096 validation
inputs. Test NMSE is summed squared difference error divided by summed
squared true difference on 32,768 independent inputs.

At \(\rho=0.6\), the NMSEs across seeds 0--4 (mean and sample standard deviation)
are \(0.1264\pm0.1774\) for \(s_T-\widehat{s}_G\) and
\(0.00923\pm0.00938\) for direct prediction. The figure plots difference
vectors at 4,096 held-out inputs from seed 0, with matching colors and
common axis limits; full-test NMSEs are \(0.0720\) and \(0.00209\), respectively.
At \(\rho=0\), the full score is affine and the indirect route instead
attains numerical precision. The example therefore illustrates a
conditional representation effect in fixed-density regression, without
directly testing the rank constraint or the denoising training procedure.

\section{FutureCredit Implementation}
\label{app:future-credit-implementation}

\paragraph{Replay and supervision.}
Each update samples \(k\) uniformly from \(\{1,\ldots,N-1\}\).
Teacher-forced replay runs through chunk \(k+1\) under a causal mask,
using its saved noisy input \(q\), solver timestep \(\tau\), and detached
history \(h\) as a differentiable input. Here \(\tau\) differs from the
supervision noise level \(t\), and \(J_k^{\mathrm{rep}}=\partial
F_{\theta,k+1}/\partial h_k\) is evaluated at the generated history.
For the flow parameterization,
\(F_\theta(q,\tau,h,c)=q-\sigma_\tau v_\theta(q,\tau,h,c)\), so its
history Jacobian already includes \(-\sigma_\tau\). Generator parameters,
saved noisy inputs, timesteps, and the teacher residual are held fixed;
the VJP target is detached without constructing a second-order graph.
The teacher clean prediction includes the configured classifier-free guidance.
Unsampled chunks are masked out of the credit loss. Each loss is averaged
over its supervised latent elements before the two are summed equally.

\paragraph{Projection and scale matching.}
For a chunk with \(m\) latent frames, the projection subtracts the frame mean,
at each channel and spatial location:
\begingroup\small
\begin{equation}
\begin{aligned}
\left(Pv\right)_f
&=v_f-\frac{1}{m}\sum_{f'=1}^m v_{f'},
\qquad f=1,\ldots,m .
\end{aligned}
\end{equation}
\endgroup
For each sample and chunk, define
\begingroup\small
\begin{equation}
\begin{aligned}
\mathcal{M}(v;a)&=\beta(v,a)v,\\
\rho(v,a)&=
\frac{\operatorname{RMS}(a)}
{\max(\operatorname{RMS}(v),10^{-8})},\\
\beta(v,a)&=
\begin{cases}
1,&\operatorname{RMS}(a)=0,\\
\min\!\left(10,\rho(v,a)\right),
&\text{otherwise},
\end{cases}
\end{aligned}
\end{equation}
\endgroup
where RMS is over all latent elements of that chunk. The actual target and
generator direction are
\begingroup\small
\begin{equation}
\begin{aligned}
c_{k,t}^{\mathrm{tgt}}
&=\mathcal{M}\!\left(
P\tilde{\mathbf c}_{k,t}^{\mathrm{VJP}},
\Delta_{k,t}^{\mathrm{tgt}}\right),\\
\widehat{g}_{k,t}
&=\mathcal{M}\!\left(b_{k,t},\widehat{r}_{k,t}\right).
\end{aligned}
\end{equation}
\endgroup
Implementation-wise, the joint estimator uses two lightweight, zero-initialized
\(1\times1\times1\) readouts from a shared bidirectional representation;
the \(C_\psi\) readout is followed by \(P\). These readouts are optimized
together under the complementary estimator losses. The credit-loss weight is
\(1\), the default update weight is \(\lambda=0.25\), and
\(w_t=\alpha_t^2/\sigma_t^2\).
The \(\lambda=0\) ablation retains the auxiliary loss and its effect on
shared features.

\end{document}